\documentclass[11pt]{article}
\usepackage{algpseudocode}
\usepackage{fullpage}
\usepackage{amsmath}
\usepackage{amsthm}
\usepackage{amssymb}
\usepackage{amsfonts}
\usepackage{hyperref}
\usepackage{color}
\usepackage{xcolor}
\usepackage{graphicx} 
\graphicspath{ {./images/} }
\usepackage[rightcaption]{sidecap}
\usepackage{wrapfig}
\usepackage{hyperref}
\newcommand{\ignore}[1]{}
\usepackage{enumitem}
\setlist{itemsep=-3pt}
\usepackage{academicons}
\definecolor{orcidlogocol}{HTML}{A6CE39}

\usepackage{algpseudocode}
\usepackage{orcidlink}
\newtheorem{lemma}{Lemma}

\newtheorem{theorem}{Theorem}

\newcommand{\A}{{\cal A}}
\newcommand{\B}{{\cal B}}

\newcommand{\E}{{\bf E}}
\newcommand{\Lin}{{\rm Lin}}
\allowdisplaybreaks
\newcommand{\al}{\allowbreak}
\newcommand{\Ex}[2]{\underset{#1}{\bf{E}}\left[#2\right]}
\newcommand{\F}{\mathbb{F}}

\newcommand{\orcid}[1]{\href{https://orcid.org/#1}{\textcolor[HTML]{A6CE39}{\aiOrcid}}}

\author{
  Nader H. Bshouty\\
  \texttt{Department of Computer Science, Technion, Israel.}\\ 
  \href{mailto:bshouty@cs.technion.ac.il}{bshouty@cs.technion.ac.il} 
  \orcidlink{0009-0007-7356-7824}
  \and
  George Haddad\\
  \texttt{Department of Computer Science, Technion, Israel.}\\ \href{mailto:haddadgeorge@campus.technion.ac.il}{haddadgeorge@campus.technion.ac.il}
  \orcidlink{0009-0002-6142-3560}
}
\date{}

\title{Approximating the Number of Relevant Variables in a Parity Implies Proper Learning}

\begin{document}

\maketitle




\begin{abstract} 
Consider the model where we can access a parity function through random uniform labeled examples in the presence of random classification noise. In this paper, we show that approximating the number of relevant variables in the parity function is as hard as properly learning parities.

More specifically, let $\gamma:{\mathbb R}^+\to {\mathbb R}^+$, where $\gamma(x) \ge x$, be any strictly increasing function. In our first result, we show that from any polynomial-time algorithm that returns a $\gamma$-approximation, $D$ (i.e., $\gamma^{-1}(d(f)) \leq D \leq \gamma(d(f))$), of the number of relevant variables~$d(f)$ for any parity $f$, we can, in polynomial time, construct a solution to the long-standing open problem of polynomial-time learning $k(n)$-sparse parities (parities with $k(n)\le n$ relevant variables), where $k(n) = \omega_n(1)$. 

In our second result, we show that from any $T(n)$-time algorithm that, for any parity $f$, returns a $\gamma$-approximation of the number of relevant variables $d(f)$ of $f$, we can, in polynomial time, construct a $poly(\Gamma(n))T(\Gamma(n)^2)$-time algorithm that properly learns parities, where $\Gamma(x)=\gamma(\gamma(x))$.

If $T(\Gamma(n)^2)=\exp({o(n/\log n)})$, this would resolve another long-standing open problem of properly learning parities in the presence of random classification noise in time~$\exp({o(n/\log n)})$. 
\end{abstract}
\pagebreak
\section{Introduction}

The problem of PAC learning parity, with and without noise, and approximating its sparsity has been extensively studied in the 
literature. See ~\cite{AustrinK14,BhattacharyyaGG16,BhattacharyyaGR15,BhattacharyyaGR20,BhattacharyyaIWX11,Blum96,BlumFKL93,BlumKW03,Bshouty18,BshoutyH98,BuhrmanGM10,ChengW10,ChengW12,DowneyFVW99,DumerMS03,Feldman07,FeldmanGKP09,GrigorescuRV11,GuijarroLR99,GuijarroTT98,Hofmeister99,KalaiMV08,KlivansS04b,Lyubashevsky05,Micciancio14,UeharaTW97,Valiant15,Vardy97} and references therein. 

In properly learning parities under the uniform distribution, the learner can observe labeled examples $\{(a_i,b_i)\}_i$, where $b_i=f(a_i)$, $a_i$ are drawn independently from the uniform distribution, and $f$ is the target parity. The goal is to return the target parity function $f$ exactly. 

In the random classification noise model with noise rate $\eta$, \cite{AngluinL87}, each label $b_i$ is independently flipped (misclassified) with probability $\eta$. 
The problem of learning parities with noise (LPN) is known to be computationally challenging. Some evidence of its hardness comes from the fact that it cannot be learned efficiently in the so-called statistical query (SQ) model~\cite{Kearns98} under the uniform distribution~\cite{BlumFJKMR94,BlumKW03}. LPN serves as the foundation for several cryptographic constructions, largely because its hardness in the presence of noise is assumed. See for example~\cite{BlumFKL93,Regev24}. 

While PAC learning of parities (and thus determining its sparsity) under the uniform distribution can be accomplished in polynomial time using Gaussian elimination, addressing this problem in the presence of random classification noise remains one of the most long-standing challenges in learning theory. 
The only known algorithm is that of Blum et al.\cite{BlumHL95}, which runs in time $2^{O(n/\log n)}$, requires $2^{O(n/\log n)}$ labeled examples, and handles only a constant noise rate. 
This algorithm holds the record as the best-known solution for this problem. Finding a $2^{o(n/\log n)}$-time learning algorithm for parities or proving the impossibility of such an algorithm remains a significant and unresolved challenge.

When the number of relevant variables\footnote{A variable is {\it relevant} in $f$ if $f$ depends on that variable.} $k$ of the parity function $f$ is known ($f$ is called $k$-sparse parity), all the algorithms proposed in the literature run in time $n^{ck}$ for some constant $c<1$, \cite{BhattacharyyaGR15,BhattacharyyaIWX11,GrigorescuRV11,Valiant15,YanYLZZ21}. 
Finding a polynomial-time algorithm for $k$-sparse parities for some $k=\omega(1)$, or proving the impossibility of such an algorithm, is another significant and unresolved challenge. 

In a related vein, another challenging problem is determining or approximating the sparsity of the parity function, i.e., the number of relevant variables in the target function. 
This problem was studied in the PAC-learning model \cite{Valiant84} under specific\footnote{Some of the problems are introduced as follows: 
Given a matrix $M\in F_2^{m\times n}$, a vector $b\in \{0,1\}^m$, and an integer $k$. 
Deciding if there exists a weight $k$ vector $x\in \{0,1\}^n$ such that $Mx=b$. 
This is equivalent to the decision problem when the distribution is uniform over the rows of $M$.} distributions~\cite{AustrinK14,BhattacharyyaBE21,BhattacharyyaGG16,BhattacharyyaIWX11,ChengW10,ChengW12,DowneyFVW99,DumerMS03,Feldman07,Micciancio14,Vardy97}. 

For the problem of determining the sparsity under any distribution and without noise, Downey et al. \cite{DowneyFVW99} and Bhattacharyya et al.~\cite{BhattacharyyaGG16} show that determining the sparsity $k$ of parities is $W[1]$-hard. 
Bhattacharyya et al. \cite{BhattacharyyaIWX11} show that the time complexity is $\min(2^{\Theta(n)},n^{\Theta(k)})$, assuming $3$-SAT has no $2^{o(n)}$-time algorithm. 
For the problem of approximating the sparsity, Dumer et al.~\cite{DumerMS03} showed that if RP$\not=$NP, then it is hard to approximate the sparsity within some constant factor $\gamma>1$. 
See also~\cite{AustrinK14,ChengW10,ChengW12,Micciancio14}. 
When the distribution is uniform, there is a polynomial-time algorithm that uses $O(n)$ labeled examples and learns parities using Gaussian elimination, thereby determining their sparsity.

In this paper, we pose the question: Can we approximate the sparsity of the parity function in polynomial time using random uniform labeled examples in the presence of random classification noise? We show that approximating the number of relevant variables in the parity function is as hard as properly learning parities. 

We first show the following. 
\begin{theorem}\label{TT1}
    Let $\gamma:{\mathbb R}^+\to {\mathbb R}^+$ be {\it any} strictly increasing function, where $\gamma(x)\ge x$. 
Consider a polynomial-time algorithm that, 
for any parity $f$, uses random uniform labeled examples of $f$ in the presence of random classification noise and returns an integer~$D$ such that\footnote{See Section~\ref{Justification} for the justification of why we use this definition and not the standard definition $d(f) \le D \le \gamma(d(f))$.} $\gamma^{-1}(d(f))\le D\le \gamma(d(f))$, where $d(f)$ is the number of relevant variables in $f$. One can, in polynomial time, construct an algorithm that runs in polynomial time, uses random uniform labeled examples in the presence of random classification noise, and learns $k(n)$-sparse parities for some\footnote{Throughout this paper, we also have $k=n-\omega(1)$. For $k=O(1)$ and $k=n-O(1)$, there are polynomial-time learning algorithms.} $k(n)=\omega_n(1)$.
\end{theorem}
 
This would solve the long-standing open problem of polynomial-time learning $k$-sparse parities for some $k = \omega_n(1)$.

We then show that 
\begin{theorem}\label{TT2}
   From any $T(n)$-time algorithm that, for any parity~$f:\{0,1\}^n\to \{0,1\}$, uses $Q(n)$  random uniform labeled examples of $f$ in the presence of random classification noise and returns a $\gamma$-approximation of the number of relevant variables $d(f)$ of $f$, one can, in polynomial time, construct a $poly(\Gamma(n))T(\Gamma(n)^2)$-time algorithm that uses $poly(\Gamma(n))Q(\Gamma(n)^2)$ random uniform labeled examples in the presence of random classification noise and properly learns parities, where $\Gamma(x)= \gamma(\gamma(x))$. 
\end{theorem}

If $T(\Gamma(n)^2)=\exp({o(n/\log n)})$, this would resolve another long-standing open problem of proper learning parities in the presence of random classification noise in time~$\exp({o(n/\log n)})$. 
This is applicable, for example, for any poly$(\cdot)$-approximation and $\exp({n^{1/c}})$-time algorithm for some sufficiently large constant $c$. As well as to quasi-poly$(\cdot)$-approximation and $\exp(\exp(\al{(\log n)^{1/c}}))$-time algorithm for some sufficiently large constant $c$.

In this paper, while the above discussions and the technique section have been primarily focused on parities, that is, linear functions over the binary field $\mathbb{F}_2$, the results we present in this paper are not limited to this specific case. We generalize our result to encompass any linear function over any finite field. This extension allows our results to be applicable to a broader range of linear systems beyond the binary paradigm, effectively widening their relevance in coding theory and cryptography.

\subsection{Our Technique}
In this section, we present the technique used in the paper to prove the results in Theorem~\ref{TT1} and~\ref{TT2}.

For learning in the presence of random classification noise, when the noise rate $\eta=1/2$, the labels will be randomly uniform, and learning is impossible. Therefore, we must assume that the learner knows some upper bound $\eta_b<1/2$ for $\eta$~\cite{AngluinL87}.
\subsection{Approximation Implies Learning $k$-Sparse Parities}
In this section, we present two approaches that prove Theorem~\ref{TT1}. The first is our method, and the second was suggested by an anonymous reviewer of RANDOM.

While the approach suggested by the anonymous reviewer is truly inspiring, we believe that our method offers significant value, is worth presenting in this paper, and may be useful for solving other problems. 

\subsubsection{First Approach}\label{MyAlg}
In this section, we will outline the technique for the binary field, though some essential details are omitted to provide a broader overview of the main concepts and approach. Additionally, proving the result for any field requires more careful treatment.

Let $\gamma:{\mathbb R}^+\to {\mathbb R}^+$ be any strictly increasing function such that for every\footnote{We need this constraint to ensure that $\gamma^{-1}(x)<\gamma(x)$ for every $x>1$.} $x>1$, $\gamma(x)> x$.  

Let $\A$ be a polynomial-time randomized algorithm that $\gamma$-approximates the number of relevant variables $d(f)$ in a parity $f$, using random uniform labeled examples of $f$ in the presence of random classification noise with any noise rate\footnote{Here, $\eta$ is not known to the algorithm, but $\eta_b$ is known.} $\eta\le \eta_b$. 
Thus, for every parity $f$ with $d(f)$ relevant variables, with probability at least $1-\delta$ we have $\gamma^{-1}(d(f))\le \A(f)\le \gamma(d(f))$. 
We will demonstrate how to construct a polynomial-time learning algorithm for $k(n)$-sparse parities, for some $k(n) = \omega_n(1)$.
First, we will show how to find $k(n)$. 

Let $\Lin(d)$ be the class of $d$-sparse parities. Assume for now that the noise rate is $\eta_b$. Later, we will show how to modify the algorithm to work for any unknown noise rate $\eta\le \eta_b$.

We first use the algorithm $\A$ to construct a table that provides values which approximate $$\Psi_\A(d)=\Ex{(f,s)\sim_u \Lin(d)\times S(f)}{\A(f)}$$
with additive error of $1/poly(n)$, for every $d\in [n]$ and noise rate $\eta_b$. 
Here $f$ is a $d$-sparse parity chosen uniformly at random from $\Lin(d)$, and $s$ is a uniformly random string in $S(f)$ - the set of random bits used by the algorithm (for the randomness of the algorithm and the noise) for which the algorithm returns a correct answer, namely, returns~$D$ such that $\gamma^{-1}(d)\le D\le \gamma(d)$. 

To approximate $\Psi_\A(d)$ for some $d\in [n]$, we iterate a polynomial number of times. 
At each iteration, we draw a random uniform $f\in \Lin(d)$ and run $\A$. For each labeled example requested by $\A$, we draw a random uniform $u\in \{0,1\}^n$ and compute $v=f(u)$. 
We then, with probability $\eta_b$, return\footnote{Here, $+$ is exclusive or.} $(u,v+1)$ to $\A$, and, with probability $1-\eta_b$, return $(u,v)$. If the algorithm outputs an integer $D$ such that $\gamma^{-1}(d)\le D\le \gamma(d)$, we retain that $D$. Otherwise, we repeat the process. 
Obviously, $\E[D]=\Psi_\A(d)$, and therefore, using Hoeffding's bound, such a table can be constructed in polynomial time. 

Now, using the fact that $\gamma$ is strictly increasing and $\gamma^{-1}(d)\le \Psi_\A(d)\le \gamma(d)$, and applying a basic averaging argument, we show that there exists a $k:=k(n)=\omega_n(1)$ for which $\Psi_\A({k+1})-\Psi_\A({k-1})\ge 1/poly(n)$. We now show how to learn $k$-sparse parities with noise rate $\eta_b$ in polynomial time and afterward for any $\eta\le \eta_b$.

Suppose that the target function $f\in \Lin(k)$ can be accessed through random uniform labeled examples in the presence of random classification noise with noise rate $\eta_b$. We first show how to approximate $\Psi_\A(d(f(x)+x_i))$ for any $i\in [n]$ without knowing $f$. Recall that $d(f(x)+x_i)$ is the number of relevant variables in $f(x)+x_i$. 
The key idea here is that if $(a,b)$ is a labeled example of~$f$, then for a random uniform permutation $\phi$, $((a_{\phi^{-1}(1)},\ldots,a_{\phi^{-1}(n)}),b+a_{i})$ is a labeled example of the function $f(x_{\phi(1)},\ldots,x_{\phi(n)})+x_{\phi(i)}$ which is a random and uniform drawn function in $\Lin(d(f(x)+x_i))$. 
Therefore, using Hoeffding's Bound, we can approximate $\Psi_\A(d(f(x)+x_i))$ for every $i\in [n]$. 

Now, $x_i$ is relevant in $f(x)$ if and only if $f(x)+x_i\in \Lin(k-1)$ and then $\Psi_\A(d(f(x)+x_i))=\Psi_\A(k-1)$. 
On the other hand, $x_i$ is not relevant in $f(x)$ if and only if $f(x)+x_i\in \Lin(k+1)$, and then $\Psi_\A(d(f(x)+x_i))=\Psi_\A(k+1)$. 
Since $\Psi_\A({k+1})-\Psi_\A({k-1})\ge 1/poly(n)$, these two cases are distinguishable in polynomial time. Consequently, we can differentiate between variables in $f$ that are relevant and those that are not. This gives the learning algorithm to~$\Lin(k)$ when the noise rate is $\eta_b$.   

This algorithm runs in time $T=poly(n,1/(1-2\eta_b))$. When $\eta$ is not known, we can run the above procedure for all possible values $\eta^{(j)}=1/2-j/T^c$, where $c$  is a sufficiently large constant, and $j\in [T^c/2]\cup\{1\}$. 
For each $j$, when the algorithm receives a labeled example $(u,v)$, we magnify the error rate to $\eta_b$ by drawing $\xi\in\{0,1\}$, which is equal to $1$ with probability $(\eta_b-\eta^{(j)})/(1-2\eta^{(j)})$, and returning $(u,v+\xi)$ to the algorithm. This new labeled example has noise rate $\eta_b$.
We collect all the $T^c/2+1$ hypotheses generated from the outputs and then employ a standard algorithm to select the one closest to the target~\cite{AngluinL87}. 
The result follows because there exists a $j$ such that\footnote{If $\eta_j=\eta+\epsilon$ then the magnified noise is $\eta_b+\lambda\epsilon$ where $\lambda=(\eta+\eta_b-1+\epsilon)/(1-2(\eta+\epsilon))$.} $|\eta^{(j)}-\eta|\le 1/T^c$. Consequently, using the total variation distance, one of the hypotheses is the target.

In this paper, we extend our result to any linear function over any finite field $\F$. The approach used is similar to the case of parities (the binary field $\F=\{0,1\}$) with some technical but nontrivial modifications.

\subsubsection{The Second Approach}
This second approach was suggested by an anonymous reviewer of RANDOM, whose insightful comments and suggestions significantly improved this manuscript for the case of the binary field. For non-binary fields, this approach can identify the relevant variables of the function. We then use the approach developed in Lemma~\ref{Existd} and Lemma~\ref{PolyLessLog} to find the coefficients of the relevant variables.

Suppose there exists a randomized algorithm $\A(n)$ that runs in time $T(n)$ and $\gamma$-approximates the number of relevant variables in a parity function $f:\{0,1\}^n\to \{0,1\}$, using random uniform labeled examples of $f$ in the presence of random classification noise with a noise rate $\eta\le\eta_b$. Here, too, the algorithm needs to know an upper bound on $\eta$. We will explain the reasons for this below. 

Let $k_1=\omega_n(1)$ be an integer such that $k_2=\gamma(\gamma(k_1)+1)=n-\omega_n(1)$. For any $k_1$-sparse parity $f$, the algorithm outputs $\A(f)\in [\gamma^{-1}(k_1),\gamma(k_1)]$, and for any $k_2$-sparse parity function $g$, it outputs $\A(g)\in [\gamma(k_1)+1,\gamma(\gamma(\gamma(k_1)+1))]$. Since the two intervals are disjoint, the algorithm can distinguish between $k_1$-sparse parities and $k_2$-sparse parities in polynomial time. Let $\B$ be the polynomial-time algorithm that distinguishes between them. 

Consider the algorithm $\B$ when it runs on random uniform examples with random uniform labels. Suppose that with probability $p$, the algorithm answers that the function is a $k_1$-sparse parity, and with probability $1-p$, it answers that it is a $k_2$-sparse parity. 
If $p>1/2$, then with probability at least $1/2$, $\B$ can distinguish between $k_2$-sparse parities and random uniform examples with random uniform labels. Otherwise, with probability at least $1/2$, $\B$ can distinguish between $k_1$-sparse parities and random uniform examples with random uniform labels. 

Suppose, without loss of generality, the latter holds. We now give an algorithm that finds the relevant variables when the target function is a $k_1$-parity function and, consequently, learns $k_1$-sparse parities. This algorithm is from~\cite{BlumFKL93}.

For every $i\in [n]$, we run the algorithm $\B$ and change the $i$-th coordinate of each example to a random uniform element in $\{0,1\}$. 
If $x_i$ is not a relevant variable of $f$, then the labeled examples are labeled examples of $f$, and the algorithm answers that it is a $k_1$-parity function. 
If $x_i$ is a relevant variable of $f$, then it is easy to see that the new labeled examples are random uniform with random uniform labels, and the algorithm answers accordingly. This distinguishes between variables in $f$ from those that are not in $f$.

In this method, too, we must know some upper bound on $\eta$. Otherwise, algorithms $\A$ and $\B$ would need to be Las Vegas algorithms, and we would not know when to stop the algorithm when dealing with random uniform examples with random uniform labels.

For the problem of finding the relevant variables of the target in other fields, the generalization of this to any field is straightforward. 

\subsection{Approximation Implies Learning Parities}
In this section, we show how $\gamma$-approximation implies proper learning parities.

Let $\gamma(x)$ be any strictly increasing function.
Suppose there exists a randomized algorithm $\A(n)$ that runs in time $T(n)$ and $\gamma$-approximates the number of relevant variables in a parity $f:\{0,1\}^n\to \{0,1\}$, using random uniform labeled examples of $f$ in the presence of random classification noise. 

Let $\Gamma(x)=\gamma(\gamma(x))$. As in Section~\ref{MyAlg}, using a basic averaging argument, we show that there exists a sequence of integers $k_1 < k_2 < \cdots < k_t < \Gamma^{-1}(n)$, where for each $i$, $k_{i+1} < \Gamma(k_i)$ and $\Psi_\A(k_i+1) - \Psi_\A(k_i-1) > {1}/poly(n)$.
As before, we obtain algorithms that learn $\Lin(k_i)$ for each $i$. 

Now, we show how to obtain a learning algorithm for 
$d$-sparse parities for every $d<\Gamma^{-1}(\sqrt{n})$. Given any $d<\Gamma^{-1}(\sqrt{n}))$, there exists a $j$ such that $k_{j-1}<d\le k_j$ where $k_0=0$. 
To learn $d$-sparse parities, we uniformly at random choose distinct $i_1,\ldots,i_{k_j-d}\in [n]$, run the algorithm for learning $k_j$-sparse parities and modify each labeled example $(a,b)$ to $(a,b+a_{i_1}+\cdots+a_{i_{k_j-d}})$. 
If $g(x)=f(x)+x_{i_1}+\cdots+x_{i_{k_j-d}}$ is in $\Lin(k_j)$, then the algorithm w.h.p learns $g(x)$. We then show that because $d<\Gamma^{-1}(\sqrt{n}))$, with high probability, the variables $x_{i_1},\ldots,x_{i_{k_j-d}}$ are not relevant variables in~$f$. We can then conclude that, w.h.p, $g\in \Lin(k_j)$. Therefore, w.h.p., we can learn $g(x)$, and consequently,  we can learn $f(x)=g(x)+x_{i_1}+\cdots+x_{i_{k_j-d}}$.

This provides a learning algorithm for $d$-sparse parities for any $d \le\Gamma^{-1}(\sqrt{n})$. Recognizing that this applies to every $n$, we can regard $f$ as a function over $N:=\Gamma(n)^2$ variables by adding $\Gamma(n)^2 - n$ dummy variables and appending $\Gamma(n)^2 - n$ random uniform elements from $\{0,1\}$ to each $a$ in the labeled example $(a,b)$. By applying this construction to the algorithm $\A(N)$, we obtain a learning algorithm for $d$-sparse parity for any $d \leq \Gamma^{-1}(\sqrt{N})=n$.

Now, the algorithm for learning parities can run all the learning algorithms for $d$-sparse parities for all $d\le n$. It takes all the outputs and then employs a standard algorithm to select the one closest to the target~\cite{AngluinL87}. See also Lemma~\ref{LearnC}. 
\pagebreak
\subsection{Justification for the Use of the $\gamma$-Approximation Definition}\label{Justification}
In our approach, we define a $\gamma$-approximation of the number of relevant variables $d(f)$ in a parity function $f$ such that $\gamma^{-1}(d(f)) \le D \le \gamma(d(f))$, instead of using the standard definition $d(f) \le D \le \gamma(d(f))$.

The key reason for this choice is that the latter definition loses its effectiveness when $d(f)$ approaches $n$, the number of variables. Specifically, if $d(f)$ is close to $n$, say $O(n)$, the condition $d(f) \le D \le \gamma(d(f))$, for $\gamma(n)=\omega(n)$, effectively reduces to $d(f) \le D \le n$. In this scenario, the value of $\gamma$ becomes less significant because the approximation $D$ would naturally fall within the trivial range of $d(f)=O(n)$ to $n$.

On the other hand, our chosen definition $\gamma^{-1}(d(f)) \le D \le \gamma(d(f))$ ensures that the approximation $D$ always depends on the function $\gamma$. This definition retains its utility even when $d(f)$ is large, as $\gamma^{-1}(d(f))$ provides a lower bound that is influenced by $\gamma$, thereby maintaining the approximation's relevance and precision.

Thus, by using the $\gamma$-approximation definition $\gamma^{-1}(d(f)) \le D \le \gamma(d(f))$, we ensure a meaningful and consistent approximation of the number of relevant variables $d(f)$ across the entire range of possible values, preserving the value and impact of the function $\gamma$. 

\section{Definitions and Preliminaries}
Let $\F$ be any finite field and $\F_q$ be the field with $q$ elements.
We define $\Lin(\F)$ as the class of all linear functions over the field~$\F$, i.e., functions $a\cdot x$ where $a\in \F^n$ and $x=(x_1,\ldots,x_n)$. A $d$-{\it sparse linear function} over $\F$ is a function in $\Lin(\F)$ with $d$ relevant variables.
The class $\Lin(\F,d)$ is the class of all $d$-sparse linear functions over $\F$. 
When $\F$ is the binary field $\F_2=\{0,1\}$, the functions in $\Lin(\F_2)$ are called parity functions, and the functions in $\Lin(\F_2,d)$ are called $d$-sparse parities. 
We use the notation $\Lin_n(\F)$ and $\Lin_n(\F,d)$ to emphasize the number of variables. 

For $f\in \Lin(\F)$, we denote by $d(f)$ the number of variables on which $f$ depends.
For a strictly increasing function $\gamma:{\mathbb R}^+\to {\mathbb R}^+$ such that $\gamma(x)> x$ for every $x$, we say that an algorithm $\A$ $\gamma$-approximates $d(f)$ in time $T=T(n)$ and $Q=Q(n)$ labeled examples if the algorithm runs in time $T$, uses $Q$ labeled examples to $f$, and with probability at least $2/3$, returns an integer $D$ such that $\gamma^{-1}(d(f))\le D\le \gamma(d(f))$. 

In proper learning $\Lin(\F)$ under the uniform distribution, the learner can observe labeled examples $(a,b)$ where $b=f(a)$ and $a\in\F^n$ are drawn independently and uniformly distributed over $\F^n$, with $f\in\Lin(\F)$ being the target linear function. 
The goal is to (properly) exactly return the linear function $f$. 
In the random classification noise model with noise rate $\eta$, each label $b$ is equal to $f(a)$ with probability $1-\eta$ and is a random uniform element in $\F\backslash \{f(a)\}$ with probability~$\eta$. 

When $\eta=1-1/|\F|$, the label is a random uniform element of $\F$; hence, learning is impossible. Therefore, we must assume that the learner knows some upper bound $\eta_b<1-1/|\F|$ for $\eta$~\cite{AngluinL87}. When $\eta=\eta_b$, to distinguish between labeled examples with random uniform labels and the function $f(x)=0$, we need at least $1/(1-\eta_b-1/|\F|)$ labeled examples. Therefore, a polynomial-time algorithm in this model is an algorithm that runs in time $poly(1/(1-\eta_b-1/|\F|)),n,1/\delta)$ \cite{AngluinL87}. 

The following Lemma shows how to learn when the algorithm has unlimited computational power.
\pagebreak
\begin{lemma}\label{LearnC}
    Let $C\subseteq \Lin(\F)$. Then $C$ is learnable under the uniform distribution in the random classification noise model in time $\tilde O(|C|\log(1/\delta)/(1-\eta_b-1/|\F|)^2)$ from $$Q=\frac{\log{\frac{|C|}{\delta}}}{(1-\eta_b-1/|\F|)^2}$$
    labeled examples.
\end{lemma}
\begin{proof}
    Let $(a,b)$ be a labeled example and $f$ be the target function. Then
    $$\Pr[f(a)=b]=\eta \Pr[f(a)=b|b\not=f(a)]+(1-\eta)\Pr[f(a)=b|b=f(a)]=1-\eta\ge 1-\eta_b.$$
    If $g\not=f$ and $g\in \Lin(\F)$ then  $$\Pr[g(a)=b]=\eta\Pr[g(a)=b|b\not=f(a)]+(1-\eta)\Pr[g(a)=b|b=f(a)]=\frac{1}{|\F|}.$$
    The result now follows by applying Chernoff's bound to estimate $\Pr[g(a)=b]$ for all $g\in C$ with confidence of $1-\delta/|C|$ and an additive error of $(1-\eta_b-1/|\F|)/4$.
\end{proof}

The following lemma shows that, in approximation algorithms, the dependency on $\delta$ is logarithmic. This is a well-known result. For completeness, a sketch of the proof is provided.

\begin{lemma}\label{BoostA}
    If there exists an algorithm $\A$ that runs in time $T(n)$, uses $Q(n)$ labeled examples to $f\in \Lin(\F,d)$ according to the uniform distribution in the presence of random classification noise and, with probability at least $2/3$, returns a $\gamma$-approximation of $d(f)$, then there is an algorithm that runs in time $O(T(n)\log(1/\delta))$, uses $O(Q(n)\log(1/\delta))$ labeled examples to $f\in \Lin(\F,d)$ according to the uniform distribution in the presence of random classification noise, and with probability at least $1-\delta$, returns a $\gamma$-approximation of $d(f)$. 
\end{lemma}
\begin{proof}
    We run $\A$,  $O(\log(1/\delta))$ times and take the median of the outputs. The correctness of this algorithm follows from an application of Chernoff's bound.
\end{proof}
The same is true for learning.
\begin{lemma}\label{BoostLe}
    If there exists an algorithm $\A$ that runs in time $T(n)$, uses $Q(n)$ labeled examples to $f\in \Lin(\F,d)$ according to the uniform distribution in the presence of random classification noise and, with probability at least $2/3$, properly learns the target $f$, then there is an algorithm that runs in time $O(T(n)\log(1/\delta))$, uses $O(Q(n)\log(1/\delta))$ labeled examples to $f\in \Lin(\F,d)$ according to the uniform distribution in the presence of random classification noise, and with probability at least $1-\delta$, learns the target $f$. 
\end{lemma}
\begin{proof}
    Since $\A$ properly learns $f$, we run the algorithm $O(\log(1/\delta))$ times and output the hypothesis that occurs most frequently in the output.
\end{proof}
\pagebreak
\section{Approximation vs. Learning}
In this section, we prove the two results.
\subsection{Approximation Implies Learning Some $\Lin(\F,k)$}
In this section, we prove that approximating the number of relevant variables in the parity function implies polynomial-time properly learning $\Lin(\F,k(n))$ for some $k(n)=\omega_n(1)$. 

We prove.
\begin{theorem}\label{TH1}
    Let $\gamma:{\mathbb R}^+\to {\mathbb R}^+$ be any strictly increasing function where $\gamma(x)>x$ for every~$x$. Let $\pi(n)$ be any function such that $\pi(n)=\omega_n(1)$. Consider any polynomial-time algorithm $\A'(n)$ that, for any linear function $f\in \Lin(\F)$, uses random uniform labeled examples of $f$ in the presence of random classification noise and, with probability at least $2/3$, returns a $\gamma$-approximation of the number of relevant variables $d(f)$ of $f$. From $\A'(n)$, one can, in polynomial time, construct a $poly(n,1/(1-\eta_b-1/|\F|),\min(|\F|,1/(1-\eta_b)^{\pi(n)}))$-time algorithm that properly learns  $\Lin(\F,k(n))$ from a polynomial number of random uniform labeled examples in the presence of random classification noise for some $k(n)=\omega_n(1)$.
\end{theorem}

We will assume for now that the noise rate $\eta=\eta_b$ is known. In Section~\ref{MyAlg}, we showed how to handle unknown noise rates $\eta\le \eta_b$. Recall that a polynomial-time algorithm in this model is an algorithm that runs in time $poly(1/(1-\eta_b-1/|\F|)),n,1/\delta)$, \cite{AngluinL87}. In particular, the algorithm constructed in Theorem~\ref{TH1} runs in polynomial time for either
\begin{itemize}
    \item Any $\eta_b$ and fields of size\footnote{This makes sense when we have a sequence of fields $\F_i$ such that $\F_i\subseteq \F_{i+1}$ and $|\F_n|=poly(n)$.} $|\F|=poly(n)$, or
    \item Any field $\F$ when $\eta_b\le 1-1/|\F|-1/\psi(n)$, where $\psi(n)=2^{o(\log(n))}$.
\end{itemize}

Let $\A'(n,s,f)$ be any algorithm that uses random uniform labeled examples of $f\in \Lin_n(\F)$ in the presence of random classification noise and, with probability at least $2/3$, returns a $\gamma$-approximation of the number of relevant variables, $d(f)$, of $f$. 
The new parameter $s$ is added for the random bits used in the algorithm for its coin flips and the noise. 
First, we will use Lemma~\ref{BoostA} to make the algorithm's success probability $1-\delta'$ for a fixed, sufficient small $\delta'$ that depends on $n$ and $|\F|$. For the proof of the Theorem in this section, $\delta'=1/(|\F|n^7)$ suffices. By Lemma~\ref{BoostA}, this adds a factor of $O(\log n+\log|\F|)$ to the time and the number of labeled examples which will be swallowed by the $\tilde O(\cdot)$ in the final time and sample complexity.
Second, we will modify the output of the algorithm to $\min(\gamma(D_f),n)$, where $D_f$ is the output of the latter algorithm. Let the resulting algorithm be denoted as $\A$. We will denote the algorithm's output by $\A(n,s,f)$. Consequently, we will have that, with probability at least~$1-\delta'$,
\begin{eqnarray}\label{DelApp}
    d(f)\le \A(n,s,f)\le \Delta(d(f))\le n
\end{eqnarray}
where $$\Delta(x)=\min(\gamma(\gamma(x)),n).$$ 

Let $S_f$ be the set of all random strings $s'$ for which $d(f) \leq \A(n,s',f) \leq \Delta(d(f))$; that is, it includes all the random strings that yield correct answers.
Throughout this section and the next, we say that $\A$ $\Delta$-approximates $d(f)$. See (\ref{DelApp}). This should not be confused with the previous definition of $\gamma$-approximates $d(f)$. Here, we use the capital letter $\Delta$ to prevent any ambiguity.

Let
$$\Psi_\A(d)=\Ex{(f,s)\sim_u \Lin(\F,d)\times S(f)}{\A(n,s,f)}$$ where $\sim_u$ indicates that $f$ is chosen uniformly at random from $\Lin(\F,d)$ and $s$ uniformly at random from $S(f)$. 
Since $d\le \A(n,s,f)\le \Delta(d)\le n$ for $s\in S(f)$ where $f\in\Lin(\F,d)$, we have 
\begin{eqnarray}\label{BoundD}
    d\le \Psi_\A(d)\le \Delta(d)\le n.
\end{eqnarray}
We note here that $\Psi_\A(d)$ is independent of $\delta$, as in $\A$, we set $\delta=\delta'$ for a fixed~$\delta'$. This is crucial for ensuring the correctness of the proof. Also, $\Psi_\A(d)$ depends on $n$. This will be essential only for the next result in the next section.

We first prove that the values of $\Psi_\A(d)$ for $d\in [n]$ can be approximated with high probability. 

\begin{lemma}\label{TablePhi} Let $0<h<1$.
    Let ${\cal A}$ be an algorithm that runs in time $T(n)$, uses $Q(n)$ labeled examples of $f\in \Lin(\F)$ according to the uniform distribution in the presence of random classification noise, and, with probability at least $1-\delta'$, $\Delta$-approximates $d(f)$. 
    A table of real values $\Psi'_\A(d)$ for $1\le d\le n$ can be constructed in time $\tilde O(n^3/h^2)T(n)\log(1/\delta)$,  and without using any labeled examples. This table, with probability at least $1-\delta$, satisfies $|\Psi'_\A(d)-\Psi_\A(d)|\le h$ for all $d\in [n]$.
\end{lemma}
\begin{proof}
    Define a random variable as the output $D$ of the algorithm $\mathcal{A}$, obtained from running it on a uniformly random $f$ from $\Lin(\F, d)$, provided that the output lies within the interval $[d, \Delta(d)]$.
    The labeled examples of $f$ can be generated by choosing a random uniform $u\in \{0,1\}^n$ and returning $(u,f(u)+e)$ to $\A$ where, with probability $1-\eta_b$, $e=0$ and, with probability $\eta_b$, $e$ is random uniform in $\F\backslash \{0\}$. 
    Obviously, $\E[D]=\Psi_\A(d)$. 
    
    By Hoeffding's bound, to compute $\E[D]$ with an additive error $h$ and a confidence probability of at least $1-\delta/(2n)$, we need to obtain $t=O((n^2/h^2)\log (n/\delta))$ values of $D$. 
    Since the success probability of obtaining a value of $D$ in the interval $[d,\Delta(d)]$ is $1-\delta'>2/3$, we need to run the algorithm $O(t+\log(2n/\delta))$ times to acquire $t$ values with a success probability at least $1-\delta/(2n)$. 
    Therefore, the time complexity is $O((t+\log(2n/\delta))nT(n))=O(tnT(n))=\tilde O(n^3/h^2)T(n)\log(1/\delta)$. 
\end{proof}

Our next result shows how to estimate $\Psi_\A(d(f))$ of the target $f$ without knowing $d(f)$.
\begin{lemma}\label{SR-L1N} Let $0<h<1$ and $\tau=O((n^2/h^2)\log(1/\delta))$.
    Let ${\cal A}$ be an algorithm that runs in time $T(n)$, uses $Q(n)$ labeled examples of $f\in \Lin(\F)$ according to the uniform distribution, in the presence of random classification noise, and, with probability at least $1-\delta'$, $\Delta$-approximates $d(f)$. 
    There is an algorithm ${\cal B}(n,h)$ that runs in time $T'=\tau T(n)$, uses $Q'=\tau Q(n)$ labeled examples of $f\in \Lin(\F)$ according to the uniform distribution in the presence of random classification noise and, with probability at least $1-\delta/2-\tau\delta'$, returns $\psi$ that satisfies $|\psi-\Psi_\A(d(f))|\le h$.
\end{lemma}

\begin{proof}
 Suppose $v\in (\F\backslash\{0\})^n$ is chosen uniformly at random and $\phi:[n]\to [n]$ is a uniformly random permutation. 
 If we run $\A$ with the target $f=\lambda_1x_{i_1}+\cdots+\lambda_dx_{i_d}$, and for every labeled example $(a,b)\in \F^n\times \F$, we modify the labeled example to $((v_1^{-1}a_{\phi^{-1}(1)},\ldots,v_n^{-1}\al a_{\phi^{-1}(n)}),b)$, then the new labeled examples remain uniform and consistent with the function $g(x)=\lambda_1v_{\phi(i_1)}x_{\phi(i_1)}+\cdots +\lambda_dv_{\phi(i_d)}x_{\phi(i_d)}$. This function $g$ is a uniformly random element of $\Lin(\F,d)$. 
 Using this fact, we show how to approximate~$\Psi_\A(d)$.

To this end, let $\tau=O((n^2/h^2)\log(1/\delta))$. 
We iterate $\tau$ times, and at each iteration, we choose a random uniform $v\in (\F\backslash \{0\})^n$ and random uniform permutation $\phi:[n]\to [n]$. 
We request for $Q(n)$ labeled examples and modify each labeled example $(a,b)\in \F^n\times \F$ to $((v_1^{-1}a_{\phi^{-1}(1)},\ldots,\al v_n^{-1} \al a_{\phi^{-1}(n)}),b)$. We then run $\A$ on these labeled examples. 
Let $D_i$ be the output of the $i$-th iteration. We then output
$\psi'=\left({\sum_{i=1}^\tau D_i}\right)/{\tau}.$

We now prove that, with probability at least $1-\delta/2-\tau\delta'$, we have $|\psi'-\Psi_\A(d(f))|\le h$. 
Since $\A(n)$ runs $\tau$ times, with probability at least $1-\tau\delta'$, all the seeds used by $\A$ are in $S(f)$ and $d(f)\le D_i\le \Delta(d(f))$. 
Also, since $\A(n)$ runs on a uniformly random function in $\Lin(\F,d)$, we have $\E[D_i]=\Psi_\A(d)$. 
By Hoeffding's bound, along with the fact that $D_i\le \Delta(d(f))\le n$, we can conclude that, with probability at least $1-\delta/2$, we have $|\psi'-\Psi_\A(d(f))|\le h$.
\end{proof}
Notice that in Lemma~\ref{SR-L1N}, $\tau$ also depends on $h$. As $h$ eventually will be $O(1/n)$ and $\delta'=1/(|\F|n^7)$, the success probability $1-\delta/2-\tau\delta'$ will be $1-o_n(1)$ for $\delta=1/n$.

We now show that in any large enough sub-interval of $[0,n]$, there is $k$ for which $\A$ can be used to learn $\Lin(\F,k)$.
\begin{lemma}\label{Existd}
    Let ${\cal A}(n)$ be an algorithm that runs in time $\allowbreak T(n)$, uses $Q(n)$ labeled examples of $f\in \Lin(\F)$ according to the uniform distribution in the presence of random classification noise, and, with probability at least $1-\delta'$, $\Delta$-approximates $d(f)$. 
    For every $1\le m\le \min\{j|\Delta(j)=n\}=\gamma^{-1}(\gamma^{-1}(n))$ there exists $m\le k\le \Delta(m)+1$ and 
    \begin{enumerate}
        \item\label{ExistdI1} An algorithm that, for every $f\in \Lin(\F,k)$, with probability at least $1-\delta/8-2\tau n\delta'$, where $\tau=O(n^4\log(1/\delta))$, identifies the relevant variables of $f$ from random uniform labeled examples in the presence of random classification noise. This algorithm runs in time $\tilde O(n^5)T(n)\log(1/\delta)$ and uses $\tilde O(n^4)Q(n)\log(1/\delta)$ labeled examples. 
        \item\label{ExistdI2} An algorithm that, with probability at least $1-\delta/2-|\F|kn\delta'$, properly learns $\Lin(\F,k)$, from random uniform labeled examples in the presence of random classification noise. This algorithm runs in time $\tilde O(|\F|kn^5)T(n)\al\log(1/\delta)$ and uses $\tilde O(n^4) Q(n)\log(|\F|/\delta)$ labeled examples.
    \end{enumerate}

    Such $k$ can be found in time $\tilde O(n^5)T(n)\log(1/\delta)$.
\end{lemma}

\begin{proof}
    We first prove the result when the field is not the binary field. Let $m$ be any integer such that $1\le m\le \min\{j|\Delta(j)=n\}$. Since by (\ref{BoundD}),
    \begin{eqnarray*}
        \sum^{\Delta(m)}_{i=m}\Psi_\A(i+1)-\Psi_\A(i)&=&\Psi_\A\left(\Delta(m)+1\right)-\Psi_\A\left(m\right)\\
        &\ge& \Delta(m)+1-\Delta(m)=1,
    \end{eqnarray*} 
there is $k$ such that $m\le k\le \Delta(m)$ and
$$\Psi_\A(k+1)-\Psi_\A(k)\ge \frac{1}{\Delta(m)-m+1}\ge  \frac{1}{n}.$$
    
First, we find such $k$. By Lemma~\ref{TablePhi}, taking $h=1/(16n)$, with probability at least $1-\delta/4$, we can find $k$ such that $\Psi_\A(k+1)-\Psi_\A(k)\ge 7/(8n)$ in time $\tilde O(n^5)T(n)\log(1/\delta)$.

We now present an algorithm that learns $\Lin(\F,k)$. 
The algorithm uses the algorithm in Lemma~\ref{SR-L1N} to approximate $\Psi_\A(d(f+x_i))$ and $\Psi_\A(d(f+\alpha x_i))$ for some $\alpha\in \F\backslash\{0,1\}$ and all $i\in[n]$ with an additive error of $1/(8n)$. 
If $x_i$ is not a relevant variable of the target function $f$, then both $f+x_i$ and $ f+\alpha x_i$ are in $\Lin(\F,k+1)$. Consequently, we obtain two values in the inteval $[\Psi_\A(k+1)-1/(8n),\Psi_\A(k+1)+1/(8n)]$. 
If $x_i$ is a relevant variable in the function, then one of the functions, either $f+x_i$ or $f+\alpha x_i$ is in $\Lin(\F,k)$, and therefore, one of the values is in the inteval $[\Psi_\A(k)-1/(8n),\Psi_\A(k)+1/(8n)]$.
Since $\Psi_\A(k)+1/(8n)<\Psi_\A(k+1)-1/(8n)$, the intervals are disjoint, and thus we can distinguish between the two cases. 

By Lemma~\ref{SR-L1N}, with probability $1-\delta/2-\tau\delta'$, we can approximate each $\Psi_\A(d(f+x_i))$ (or $\Psi_\A(d(f+\alpha x_i))$) with an additive error $h=1/(8n)$ in time $\tau T(n)$ and $\tau Q(n)$ labeled examples, where $\tau=O(n^4\log(1/\delta))$. Taking $\delta/(8n)$ for $\delta$, with probability $1-\delta/8-2\tau n\delta'$, we can approximate all $\Psi_\A(d(f+x_i))$ and $\Psi_\A(d(f+\alpha x_i))$, $i\in [n]$ with an additive error $h=1/(8n)$ in time $\tau' nT(n)$ and $\tau' Q(n)$ labeled examples where $\tau'=O(n^4\log(n/\delta))$.
This completes the proof of item~\ref{ExistdI1} for the case where the field is not the binary field.

To prove item~\ref{ExistdI2}, suppose, without loss of generality, that $x_1,\ldots,x_k$ are the relevant variables in $f$.
We approximate $\psi_{\alpha,i}:=\Psi_\A(d(f-\alpha x_i+x_{k+1}))$ for all $\alpha\in \F$ and for every~$i\in [n]$. 
The result follows from the fact that $\psi_{\alpha,i}=\Psi_\A(k)$ if and only if the coefficient of $x_i$ is $\alpha$. Otherwise, $\psi_{\alpha,i}=\Psi_\A(k+1)$. 
By Lemma~\ref{SR-L1N}, to approximate all $\psi_{\alpha,i}$, with success probability of $1-\delta/4-|\F|kn\delta'$, we need time $O(|\F|kn^5T(n)\log(|\F|n/\delta))$ and $O(n^4 T(n)\log(|\F|n/\delta))$ labeled examples. This completes the proof of item~\ref{ExistdI2} for the case where the field is not the binary field. 

Similar to the approach described above, for the binary field, we can show that there exists a $k$ such that $\Psi_\A(k+1) - \Psi_\A(k-1) \ge 1/n$. 
Then, use the algorithm described in Lemma~\ref{SR-L1N} to approximate $\Psi_\A(d(f+x_i))$ for all $i \in [n]$. 
If $x_i$ is not a relevant variable of the target function $f$, then $\Psi_\A(d(f+x_i))\in [\Psi_\A(k+1)-1/(8n),\Psi_\A(k+1)+1/(8n)]$. 
If $x_i$ is a relevant variable in $f$, then $\Psi_\A(d(f+x_i))\in[\Psi_\A(k-1)-1/(8n),\Psi_\A(k-1)+1/(8n)]$. Since both intervals are disjoint, we obtain the desired result.\footnote{Another approach is to utilize the fact that $\Psi_\A(k)-\Psi_\A(k-1) \geq 1/n$, and for every $i$, approximate $\Psi_\A(g_i)$, where $g_i = f(x_1, \ldots, x_{i-1}, 0, x_{i+1}, \ldots, x_n)$, using all labeled examples $(a, b)$ that satisfy $a_i = 0$.} 
\end{proof}
Notice that in item~\ref{ExistdI2}, the success probability $1-\delta/2-|\F|kn\delta'$, and the time complexity depends on $|\F|$. We now present an alternative algorithm for finding the coefficients of the linear function, given that the algorithm knows the relevant variables.  

\begin{lemma}\label{PolyLessLog}
   Let $\A$ be an algorithm that, for every $f\in \Lin(\F,k)$, runs in time $T$, uses~$Q$ random uniform labeled examples in the presence of random classification noise, and identifies the relevant variables of $f$. 
   Then
    there is an algorithm that properly learns $\Lin(\F,k)$ in time $T+\tilde O((n^3/(1-\eta_b)^{k}+n/(1-\eta_b-1/|\F|)^2)\log(1/\delta))$ and uses $Q+O(((1/(1-\eta_b)^{k}+n/(1-\eta_b-1/|\F|)^2)\log(1/\delta))$ random uniform labeled examples in the presence of random classification noise.
\end{lemma}
\begin{proof}
    We run $\A$ to find the relevant variables.  
    The algorithm that finds the coefficients iterates $O((1/(1-\eta_b)^k)\log(1/\delta))$ times. At each iteration, it requests $k$ labeled examples and uses Gaussian elimination to find the coefficients. Then, it tests whether the output function matches the target. If not, it proceeds to the next iteration.

    Since $\eta=\eta_b$, the probability that all the labeled examples are correct is $(1-\eta_b)^k$. If $|\F|=q$, the probability that $k$ random entries in the examples form a non-singular matrix is
    $$\left(1-\frac{1}{q^k}\right)\left(1-\frac{1}{q^{k-1}}\right)\cdots\left(1-\frac{1}{q}\right)\ge \frac{1}{4}.$$
    Therefore, after $t=O((1/(1-\eta_b)^k)\log(1/\delta))$ iterations, with probability at least $1-\delta/3$, at least one of the outputs is the target. By Lemma~\ref{LearnC}, learning the target from a set of $t$ linear functions with a success probability of $1-\delta/3$ requires $\tilde O((1/(1-\eta_b-1/q)^2)(k+\log(1/\delta))$ labeled examples.  
\end{proof}

We are now ready to prove Theorem~\ref{TH1}. 
\begin{proof}
Let $\A'$ be an algorithm that runs in polynomial time, uses labeled examples according to the uniform distribution in the presence of random classification noise, and outputs $\gamma^{-1}(d(f))\le D\le \gamma(d(f))$. 
Modify the algorithm to output $\min(\gamma(D),n)$. The algorithm now is a $\Delta$-approximation algorithm, where $\Delta(x)=\min(\gamma(\gamma(x)),n)$. 
Let $m(n)=\gamma^{-1}(\gamma^{-1}(\pi( n)))$. 
Since $\gamma:{\mathbb R}^+\to {\mathbb R}^+$ is strictly increasing and is defined for all ${\mathbb R}^+$, we have $\gamma^{-1}:{\mathbb R}^+\to {\mathbb R}^+$, is also strictly increasing, and $m(n)=\omega_n(1)$. 
Let $\delta'=1/(|\F|n^7)$ and $\delta=1/n$. By Lemma~\ref{Existd}, item~\ref{ExistdI1}, there exists $m(n)\le k(n)\le \Delta(m(n))=\pi(n)$, and an algorithm that, for every $f\in \Lin(\F,k(n))$, with probability at least $1-o_n(1)>2/3$, identifies the relevant variables of $f$ from random uniform labeled examples in the presence of random classification noise. 
Also, this algorithm runs in polynomial time. Since $k(n)\le \pi(n)$, by Lemma~\ref{PolyLessLog} and item~2 in Lemma~\ref{Existd}, there is a $poly(n,1/(1-\eta_b-1/|\F|),\min(|\F|,1/(1-\eta_b)^{\pi(n)}))$ time learning algorithm for $\Lin(\F,k)$.
Since $k(n)\ge m(n)=\omega_n(1)$, the result follows.
\end{proof}

\subsection{Approximation Implies Learning $\Lin(\F)$}
In this section, we prove.
\begin{theorem}\label{TH2}
    Let $\gamma:{\mathbb R}^+\to {\mathbb R}^+$ be any strictly increasing function, where $\gamma(x)>x$ for every~$x$. Let $\Gamma(x):=\gamma(\gamma(x))$. 
    Consider any $T(n)$-time algorithm $\A'(n)$ that, for any linear function $f\in \Lin(\F)$, uses $Q(n)$ random uniform labeled examples of $f$ in the presence of random classification noise and, with probability at least $2/3$, returns a $\gamma$-approximation of the number of relevant variables $d(f)$ of $f$.
    From $\A'(n)$, one can, in polynomial time, construct a $\tilde  O(|\F|\Gamma(n)^{12})T(O(\Gamma(n)^2))\log(1/\delta)$-time algorithm that properly learns  $\Lin(\F)$ from $\tilde  O(\Gamma(n)^8)Q(O(\Gamma(n)^2))\log(|\F|/\delta)$ random uniform labeled examples in the presence of random classification noise.
\end{theorem}

We first show that for every $d$ satisfying $12 \Gamma(d)^2\le n$, there exists a learning algorithm for $\Lin(\F,d)$.
\begin{lemma}\label{LessQ}
    Suppose that for every $1\le m\le m':=\Gamma^{-1}(n)$, there exists a $k$ such that $m\le k\le \Gamma(m)$, and an algorithm that runs in time $T(n)$ and, with probability at least $2/3$, properly learns $f\in \Lin(\F,k)$ under the uniform distribution in the presence of random classification noise, and uses $Q(n)$ labeled examples. 
    Then, for every $d$ such that $12\Gamma(d)^2\le n$, there is an algorithm that runs in time $O(T(n)\log (1/\delta))$ and properly learns $\Lin(\F,d)$ under the uniform distribution in the presence of random classification noise, using $O(Q(n)\al \log (1/\delta))$ labeled examples.
\end{lemma}
\begin{proof} 
    Let $d$ be an integer such that $12\Gamma(d)^2<n$. 
    Since $\Gamma(d)<n$, we have $d\le m'$. Consequently, there exists $k$ such that $d\le k\le \Gamma(d)\le \Gamma(m')=n$, along with a proper learning algorithm $\B(n,k)$ for $\Lin(\F,k)$, that runs in time $T(n)$ and uses $Q(n)$ labeled examples. 
    
    We now present an algorithm for learning $\Lin(\F,d)$. 
    We uniformly at random draw $k-d$ variables $x_{i_1},\ldots,x_{i_{k-d}}$ and run the algorithm $\B(n,k)$. 
    For each labeled example $(a,b)$ of $f$, we feed $\B$ with the labeled example $(a,b+a_{i_1}+\cdots+a_{i_{k-d}})$. 
    This modified labeled example serves as a labeled example for the function $g=f+x_{i_1}+\cdots+x_{i_{k-d}}$. 
    The probability that $g\in \Lin(\F,k)$ is the probability that none of the variables $x_{i_1},\ldots,x_{i_{k-d}}$ are relevant in $f$. This probability is given by 
    $$\prod_{i=d}^k\left(1-\frac{i}{n}\right)\ge 1-\frac{k^2}{n}\ge 1-\frac{\Gamma(d)^2}{n}\ge \frac{11}{12}.$$
    
    Therefore, with probability at least $1-(1/3+1/12)>1/2$, algorithm $\B(n,k)$ learns $g$ and thus learns~$f$. 
    By Lemma~\ref{BoostLe}, the result follows.
\end{proof}

We now show how to construct a learning algorithm for $\Lin(\F,d)$ for every $d\le n$.

\begin{lemma}\label{ALLn}
    Suppose that for every $n$ and every $d$ that satisfies $12\Gamma(d)^2\le n$, there is an algorithm $\A(n)$ that runs in time $T(n)$ and, with probability at least $2/3$, properly learns $f\in \Lin(\F,d)$ under the uniform distribution in the presence of random classification noise, using $Q(n)$ labeled examples. 
    Let $N(n)=12\Gamma(n)^2$. 
    Then, for every $n$ and every $d\le n$, there is an algorithm that runs in time $T(N(n))$ and, with probability at least $2/3$, properly learns $f\in \Lin(\F,d)$ under the uniform distribution in the presence of random classification noise, using $Q(N(n))$ labeled examples.
\end{lemma}
\begin{proof} 
    Let $N=N(n)$. Then for every $d\le n$, we have $12\Gamma(d)^2\le 12\Gamma(n)^2= N(n)$. We run $\A(N)$. Whenever the algorithm requests a labeled example, we draw a labeled example $(a,b)\in \F^n\times \F$, append $N-n$ random uniform entries to $a$, creating~$a'$. We then provide $(a',b)$ to $\A(N)$. The algorithm is effective for any $d$ that satisfies $12\Gamma(d)^2\le N=12\Gamma(n)^2$ and, thereby, covers all $d\le n$.
\end{proof}

Theorem~\ref{TH2} now follows from Lemma~\ref{Existd}, \ref{LessQ} and~\ref{ALLn}.
\\
\\

{\bf Acknowledgements:}{We would like to thank the anonymous reviewer of RANDOM for providing another approach for finding the relevant variables in the target function. We also extend our gratitude to the other reviewers for their useful comments and suggestions, which have greatly improved this manuscript.}

\bibliography{TestingRef}

\end{document}